\documentclass[letterpaper, 10 pt, conference]{ieeeconf}
\IEEEoverridecommandlockouts                              
\let\labelindent\relax
\usepackage{amsfonts}       

\usepackage{amsthm}
\usepackage{mathtools}      
\usepackage{amssymb}        
\usepackage{filecontents}
\usepackage{graphicx}       
\usepackage{marginnote}     
\usepackage{marvosym}       
\usepackage{overpic}        
\usepackage{tabularx}
\usepackage{cite}
\usepackage{color}
\usepackage[linesnumbered,algoruled,boxed,lined]{algorithm2e}
\usepackage[normalem]{ulem}
\usepackage[textsize=footnotesize]{todonotes}
\usepackage{enumitem}
\usepackage[normalem]{ulem}
\usepackage{bmpsize}
\newtheorem{problem}{Problem}
\newtheorem{proposition}{Proposition}[section]

\theoremstyle{definition}

\theoremstyle{remark}
\newtheorem*{remark}{Remark}

\SetKwProg{Fn}{Function}{}{}
\SetKwComment{Comment}{$\triangleright$\ }{}

\newcommand{\changed}[1]{\textcolor{black}{#1}}

\newcommand{\mpp}{\textsc{MPP}\xspace}
\newcommand{\mcr}{\textsc{MCR}\xspace}
\newcommand{\mmcr}{\textsc{MMCR}\xspace}
\newcommand{\qcop}{\textsc{QCOP}\xspace}
\newcommand{\otp}{\textsc{OTP}\xspace}
\newcommand{\rcp}{\textsc{RCP}\xspace}

\makeatletter
\def\subsubsection{\@startsection{subsubsection}
                                 {3}
                                 {\z@ \hspace*{1mm}}
                                 {0ex plus 0.1ex minus 0.1ex}
                                 {0ex}
                                 {\normalfont\normalsize\itshape}}
\makeatother

\title{
Integer Programming as a General Solution Methodology for Path-Based Optimization in Robotics: Principles, Best Practices, and Applications
}
\author{
Shuai D. Han \quad Jingjin Yu
\thanks{
S. D. Han and J. Yu are with the Department of Computer Science, 
Rutgers, the State University of New Jersey, Piscataway, NJ, USA. E-Mails: 
\{{\tt shuai.han, jingjin.yu}\}\hspace*{.25em}\MVAt \hspace*{.25em}rutgers.edu. 
}%
}

\begin{document}
\maketitle
\thispagestyle{empty}
\pagestyle{empty}

\begin{abstract}
Integer programming (IP) has proven to be highly effective in solving many 
path-based optimization problems in robotics. However, the applications of
IP are generally done in an ad-hoc, problem specific manner. In this work, 
after examined a wide range of path-based optimization problems, we describe 
an IP solution methodology for these problems that is both easy to apply (in 
two simple steps) and high-performance in terms of the computation time and 
the achieved optimality. 
We demonstrate the generality of our approach through the application to 
three challenging path-based optimization problems: {\em multi-robot path 
planning} (MPP), {\em minimum constraint removal} (MCR), and {\em reward 
collection problems} (RCPs). 
Associated experiments show that the approach can efficiently produce 
(near-)optimal solutions for problems with large state spaces, complex 
constraints, and complicated objective functions.
In conjunction with the proposition of the IP methodology, we introduce 
two new and practical robotics problems: {\em multi-robot minimum 
constraint removal} (MMCR) and {\em multi-robot path planning} (MPP) with 
{\em partial solutions}, which can be quickly and effectively solved using 
our proposed IP solution pipeline. 
\end{abstract}

\section{introduction}\label{sec:introduction}
The study of robot task and motion planning problems aims at finding a 
path (resp., paths) for the robot (resp., robots) to optimize certain 
cumulative cost or reward. While some settings admit efficient 
search-based algorithmic solutions, e.g., via dynamic programming, such 
problems are frequently computationally intractable 
\cite{hauser2014minimum,Yu2015IntractabilityPlanar}. In such cases, two 
approaches are often employed: {\em (i)} designing polynomial-time 
algorithms that compute approximately optimal solutions, and {\em (ii)} 
applying greedy search, assisted with heuristics. Both approaches have 
their fair share of drawbacks in practical applications: the former does 
not always ensure good optimality and the later often does not scale well 
as the problem instance becomes larger. 

In this paper, we describe an {\em integer programming} (IP) methodology 
as a third general solution approach toward challenging path-based 
optimization problems. The key to building an IP-based solution is the 
construction of a {\em model} constituting of variables and 
inequalities that encodes all constraints of the target problem. For 
optimizing over paths, we make the important observation that it is 
natural to partition the model construction process into a {\em 
path-encoding} step followed by a second step that adds the {\em 
optimization constraints}. Following this methodology, we can readily 
solve many distinct and challenging path-based optimization problems 
including multi-robot path planning (MPP), minimum constraint removal 
(MCR), and reward collection problems (RCPs). As shown with extensive 
evaluation, the IP approaches often come with competitive performance 
in terms of both computation time and solution optimality. 
In conjunction with the proposition of the IP methodology, we introduce 
two new robotics problems: {\em multi-robot minimum constraint removal} 
(MMCR) and {\em multi-robot path planning} (MPP) with {\em partial 
solutions}. These problems are natural generalizations of MCR and MPP, 
respectively, that are practical but can be more challenging 
computationally. 

\noindent\textbf{Related Work}.
Integer programming (IP) methods are widely used to tackle 
combinatorial optimization challenges since their 
inception~\cite{dantzig2016linear,nemhauser1988integer}, with 
applications to a variety of problems spanning the traveling 
salesperson problem (TSP)~\cite{miller1960integer}, network 
flow~\cite{hu1969integer}, multi-target 
tracking~\cite{morefield1977application}, etc. More recent 
studies have applied IP on path optimization problems in robotics 
including multi-robot path planning \cite{schouwenaars2001mixed,
peng2005coordinating,YuLav16TRO} and robotic 
manipulation~\cite{ding2011mixed,han2017complexity}, to list a few. 

This work is motivated by and builds on a long line of work that 
used IP, starting with the surprising initial success as IP was 
applied to multi-robot path planning (MPP) \cite{YuLav16TRO}, which 
achieved a leap in performance in optimally solving MPP. MPP is an 
important problem that finds applications in a diverse array of 
areas including evacuation~\cite{RodAma10}, formation~\cite{PodSuk04,
SmiEgeHow08}, localization~\cite{FoxBurKruThr00}, microdroplet 
manipulation \cite{GriAke05}, object transportation~\cite{RusDonJen95}, 
search and rescue~\cite{JenWheEva97}, and human robot 
interaction~\cite{knepper2012pedestrian}. In the past decade, 
significant progress has been made on optimally solving MPP problems 
in discrete, graph-base environments. Algorithmic solutions for such 
a discrete MPP are often through reduction to other problems 
\cite{Sur12,erdem2013general,YuLav16TRO}. Decoupling-based heuristics 
are also proven to be useful~\cite{StaKor11,wagner2015subdimensional,
boyarski2015icbs}. Similar to the partial solution aspect examined in 
this paper, a recent work~\cite{MaIJCAI18} provides a search-based 
solver which optimizes the number of robots that reach goals in a 
limited time horizon.

To demonstrate the generality and ease of application of our 
methodology, we also examined the minimum constraint removal (MCR) 
and a class of robotic reward collection problems (RCPs). MCR, which 
requires finding a path while removing the least number of blocking 
obstacles, is relevant to constraint-based task and motion 
planning~\cite{lozano2014constraint,dantam2018incremental}, object 
rearrangement \cite{krontiris2016efficiently,han2017complexity}, and 
control strategy design~\cite{castro2013incremental}. Two search-based 
solvers are provided in~\cite{hauser2014minimum} that extend to 
weighted obstacles~\cite{garrett2015ffrob}. Methods exist that balance 
between optimality, path length and computation 
time~\cite{krontiris2015computational}. Recent studies on MCR reduce 
the gap between lower and upper bounds regarding 
optimality~\cite{eiben2018improved}. Reward collection problems (RCPs) 
are generally concerned with gathering rewards without exceeding some 
(e.g., time or distance) budget. There is a variety of such problems 
including the classical traveling salesperson problem 
(TSP)~\cite{lawler1985traveling} and the orienteering problem 
(OP)~\cite{vansteenwegen2011orienteering}. Our focus here is geared 
toward the more complex variations~\cite{yu2016correlated,yu2014optimal} 
involving non-additive optimization objectives. Such problems model 
challenging information gathering tasks, e.g., precision 
agriculture~\cite{tokekar2016sensor}, monitoring environmental 
attributes of the ocean~\cite{ma2016information}, and infrastructure 
inspection~\cite{papachristos2016distributed}. These problems are 
generally at least NP-hard \cite{GarJoh79,vansteenwegen2011orienteering,
hauser2014minimum}.


\noindent\changed{
\textbf{Contributions}. This study brings two main contributions: 
}
\vspace*{-1mm}
\begin{itemize}[leftmargin=.15in]
\item Building on previous studies, we propose a general integer 
programming (IP) solution framework for path-based optimization 
problems in robotics. The two-step pipeline of the framework is 
easy to apply and also frequently produces highly optimal solutions.
\item We formulate the multi-robot minimum constraint removal (MMCR) 
problem and the multi-robot path planning (MPP) with partial 
solutions problem, as practical generalizations of MCR and MPP, 
respectively. We show that the IP approach can effectively solve 
these new problems. 
\end{itemize}
In addition to the main contributions, the study provides many 
additional, problem-specific heuristics that significantly enhance
the performance of the baseline IP formulation; some of these 
heuristics are also generally applicable. Unless explicitly mentioned 
and referenced, the enhancements (e.g., heuristics) described in the 
paper are also presented here for the first time. While IP-based methods 
have already been used in robotics, to the best of our knowledge, this 
work is the first one that summarizes a general IP framework that can be 
readily applied to a variety of path-based optimization problems, backed
by thorough simulation-based experimental evaluations. 

\noindent\textbf{Organization}. 
The rest of this paper is structured as follows. 
In Section~\ref{sec:preliminaries}, we formally define MPP, multi-robot MCR, and RCP. 
In Section~\ref{sec:method}, we outline the general IP solution methodology and 
describe two tried-and-true approaches for path encoding. 
In Sections~\ref{sec:applications}--\ref{sec:rcp}, we demonstrate how the IP model may be 
completed for the three diverse robotics problems and introduce many best practices 
along the way. We conclude in Section~\ref{sec:conclusion}. 
\changed{In Appendix, we provide a guidance to IP implementation.}

\vspace*{-2mm}
\section{Preliminaries}\label{sec:preliminaries} 
\vspace*{-1mm}
\subsection{Path-Based Optimization Problems}
For stating path-based optimization problems, we adopt the standard graph-theoretic 
encoding of paths. Consider a connected undirected graph $G(V, E)$ with vertex set $V$ 
and edge set $E$. For $v_i \in V$, let $N(v_i) = \{v_j \mid (v_i, v_j) \in E\}$ be the 
{\em neighborhood} of $v_i$. For a robot with initial and goal vertices $x^I, x^G \in V$, 
a {\em path} is defined as a sequence of vertices $P = (p^0, \dots, p^T)$ satisfying:
{\em (i)} $p^0 = x^I$; {\em (ii)} $p^T = x^G$; {\em (iii)} $\forall 1 \leq t \leq T$, 
$p^{t - 1} = p^t$ or $(p^{t - 1}, p^t) \in E$. For $n$ robots with their initial and 
goal configurations given as $X^I = \{x_1^I, \dots, x_n^I\}$ and $X^G = \{x_1^G,
\dots, x_n^G\}$, 
the paths are then $\mathcal P = \{P_1, \dots, P_n\}$, where $P_i = (p_i^0, \dots, p_i^T)$. 
These paths are not necessarily collision-free. We now outline three diverse 
classes of path-based optimization problems. 

\subsubsection{Multi-Robot Path Planning (MPP)} 
The main task in multi-robot path planning (MPP) is routing robots to 
their goals while avoiding robot-robot collisions, which happen when 
two robots meet at a vertex or an edge. Note that the graph-theoretic 
formulation already considers static obstacles. For $\mathcal P$ to be 
collision-free, $\forall 1 \leq t \leq T$, $P_i, P_j \in \mathcal P$ 
must satisfy: {\em (i)} $p_i^t \neq p_j^t$; {\em (ii)} $(p_i^{t - 1}, 
p_i^t) \neq (p_j^t, p_j^{t - 1})$. The objective for MPP is to minimize 
the {\em makespan} $T$, which is the time for all the robots to reach 
the goal vertices. 

In this paper, we also introduce a practical MPP generalization that 
allows {\em partial solutions}, i.e., only $k \le n$ robots are required 
to reach their pre-specified goals. This formulation models
diminishing reward scenarios where the payoff stops accumulating 
after a certain amount of targets are reached. Here, the ending vertices 
for the other $(n - k)$ paths can be arbitrary vertices in $V$. Note 
that $\mathcal P$ must still be collision-free. Being more general, the 
problem is also much more difficult to solve. We note that this setting
is different from the case where the robots are indistinguishable, which 
is much simpler and admits fast polynomial time algorithms. 

\begin{problem}[Generalized Time-Optimal MPP]\label{prob:mpp}
Given $\langle G, X^I, X^G, k \rangle$, find a collision-free path set $\mathcal P$ 
that routes at least $k$ robots to the goals and minimizes $T$.
\end{problem}

\begin{figure}[ht!]
    \centering
    \begin{overpic}[keepaspectratio, scale = 0.75]{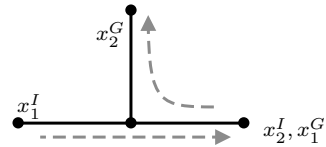}
        \put(35, 45){\footnotesize $x^G_2$}
        \put(2, 15){\footnotesize $x^I_1$}
        \put(105, 5){\footnotesize $x^I_2, x^G_1$}
    \end{overpic}
    \caption{\changed{A \mpp example with $V$ and $E$ colored in black. 
    When $k = n = 2$, a $3$ step collision-free min-makespan solution requires robot $1$
    to stay still in the beginning as robot $2$ moves to the middle vertex. 
    When $k = 1$, the $2$ step optimal solution only moves robot $2$ to its goal.}}\label{fig:mpp-example}
\end{figure}

\changed{An example for (partial) \mpp is provided in Fig.~\ref{fig:mpp-example}.}

\noindent\subsubsection{Multi-Robot Minimum Constraint Removal (MMCR)} 
Given a graph $(V, E)$, let an {\em obstacle} $O \subset V$ be a subset 
of vertices in $V$. Given a finite set of obstacles $\mathcal O = \{O_1, 
\dots\}$, the multi-robot minimum constraint removal problem seeks a 
solution $\mathcal P$ and a set of obstacles to be removed $\mathcal O_r 
\subset \mathcal O$, such that paths in $\mathcal P$ do not traverse 
through any obstacles in $\mathcal O \backslash \mathcal O_r$. The 
objective is to minimize the number of obstacles to be removed, 
i.e., $|\mathcal O_r|$. More formally: 

\begin{figure}[ht!]
    \centering
    \vspace{1mm}
    \begin{overpic}[keepaspectratio, width = 0.95\linewidth]{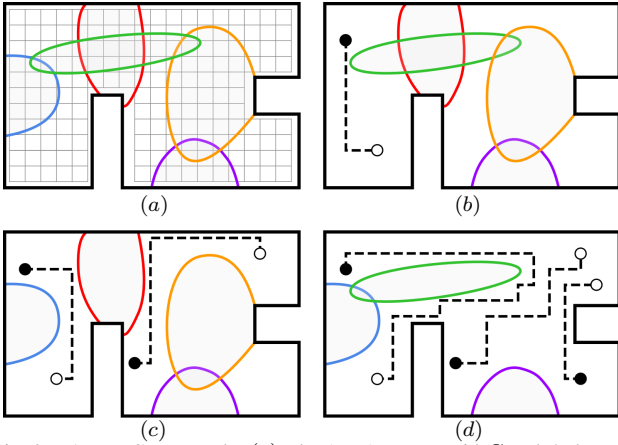}
        \put(22, 33.5){\footnotesize $(a)$}
        \put(73, 33.5){\footnotesize $(b)$}
        \put(22, -3){\footnotesize $(c)$}
        \put(73, -3){\footnotesize $(d)$}
    \end{overpic}
    \vspace{1mm}
    \caption{\changed{A \mmcr example. $(a)$ The (gray) square grid $G$ and 5 obstacles colored 
    in red, blue, green, orange and purple. $G$ is not drawn in the other sub-graphs for better 
    visibility. $(b)$ An optimal solution which removes the blue 
    obstacle when $n = 1$. The dashed line shows a feasible path for the robot to move from 
    its start (the circle) to the goal (the black dot). An alternative optimal solution is to 
    remove the green obstacle. $(c)$ The only optimal solution when $n = 2$. $(d)$ An optimal 
    solution when $n = 3$; alternative optimal solutions are also avaliable.}}\label{fig:mmcr-example}
\end{figure}

\begin{problem}[MMCR]
Given $\langle G, X^I, X^G, \mathcal O \rangle$, find $\mathcal P$ and $\mathcal O_r$ 
which minimizes $|\mathcal O_r|$, and for all 
$P_i \in \mathcal P, O \in \mathcal O \backslash \mathcal O_r, 0 \leq t \leq T$: 
$p_i^t \notin O$.
\end{problem}
\vspace*{-1mm}

\changed{An illustration of \mmcr is provided in Fig.~\ref{fig:mmcr-example}.}
\subsubsection{Reward Collection Problem (RCP)} 
Denote $\mathbb R_{\geq 0}$ as the set of non-negative real numbers. In a reward collection 
problem (RCP), a robot is tasked to travel on a graph $G$ for a limited amount of time 
$c^* \in \mathbb R_{\geq 0}$, while maximizing the reward it collects\footnote{Multi-robot 
RCP can be readily defined; we omit the such discussions given the already 
extensive multi-robot coverage and the page limit.}. Such a setup defines two functions: a cost 
function $C: P \to \mathbb R_{\geq 0}$ specifies the time already spent, and a reward 
function $R: P \to \mathbb R_{\geq 0}$ tracks the reward collected along $P$. Obviously,
{\em time} here may be replaced by other types of bounded resources, e.g., fuel. 

\vspace*{-1mm}
\begin{problem}[RCP]
Given $\langle G, x^I, x^G, C, R, c^* \rangle$, find $P$ that maximize $R(P)$ under the 
constraint $C(P) \leq c^*$,
\end{problem}
\vspace*{-1mm}

We work with two variations of RCP in this paper: the quadratic correlated orienteering 
problem (QCOP) and the optimal tourist problem (OTP).

In a {\em Quadratic Correlated Orienteering Problem (QCOP)}, 
each vertex $v_i \in V$ is associated with a reward $r_i \in \mathbb R_{\geq 0}$. 
If $v_i \in P$, not only $r_i$ but also partial rewards from vertices in $N(v_i)$ 
are collected if $P$ does not contain these vertices. Denote the correlated weights as 
$\{w_{ij} = 1 / |N(v_j)| \ | v_j \in N(v_i)\}$, the maximum possible reward collected 
from $v_i$ is $r_i + \sum_{v_j \in N(v_i)} w_{ij}r_j$. Suppose $x_i \in \{0, 1\}$ indicates 
whether $v_i \in P$, the total reward collected is 
\begin{align}\label{qcop-r}
R_{QCOP}(P) = \sum_{i = 1}^{|V|} 
(r_i x_i + \sum_{v_j \in N(v_i)} r_j w_{i j} x_i (x_j - x_i)).
\end{align} 
In QCOP, each edge $(v_i, v_j) \in E$ is associated with a time cost 
$c_{ij} \in \mathbb R_{\geq 0}$. The time constraint of QCOP requires that 
\begin{align}\label{qcop-t}
C_{QCOP}(P) = \sum_{t = 1}^T c_{p^{t - 1}, p^t} \leq c^*.
\end{align}

In an {\em Optimal Tourist Problem (OTP)}, 
the reward collected at vertex $v_i \in V$ is described by a non-decreasing function 
$R_i(t_i)$, where $t_i$ is the length of time that the robot spends at $v_i$. 
The total reward is then expressed as 
\begin{align}\label{otp-r}
R_{OTP}(P) = \sum_{v_i \in P} R_i(t_i).
\end{align} 
Similar to QCOP, the time constraint of OTP requires that 
\begin{align}\label{otp-t}
C_{OTP}(P) = \sum_{t = 1}^T c_{p^{t - 1}, p^t} + \sum_{v_i \in P} t_i \leq c^*.
\end{align}


It has been shown that optimal MPP, QCOP, and OTP are all NP-hard 
\cite{Yu2015IntractabilityPlanar,yu2016correlated,yu2014optimal}. MMCR is a 
multi-robot generalization of the single-robot MCR problem, which is known 
to be NP-hard \cite{hauser2014minimum}. As a consequence, MMCR is also 
computationally intractable.

\subsection{Integer Programming Basics} 

Integer programming (IP), roughly speaking, addresses a class of 
problems that optimize an objective function subject to integer 
constraints. A typical IP sub-class is integer linear programming 
(ILP). Given a vector $\mathbf{x}$ consists of integer variables, 
a general ILP model is expressed as:
\vspace*{1mm}
\begin{align*}
    \text{minimize} \,\,\, \mathbf{c}^T \mathbf{x} \quad 
		\text{subject to} \,\,\, A \mathbf{x} \leq \mathbf{b}.
\end{align*}

Here, $\mathbf{c}, \mathbf{b}$ are vectors and $A$ is a matrix. 
In general, solution to $\mathbf{x}$ must contain only integers. 
Note that the formulation is compatible with equality constraints 
since an equality constraint can be interpreted as two inequality 
constraints. 
As will be demonstrated, due to IP's rather straightforward 
formulation, the reduction from path-planning problems to IP are often 
not hard to achieve in low polynomial time using our methodology. 
Apart from the unsophiscated reductions, IP is also a well-known 
fundamental methematical problem that has been studied extensively. 
Thus, an IP model can often be solved quickly and optimally using 
solvers like Gurobi~\cite{gurobi}, Cplex~\cite{cplex} and 
GLPK~\cite{glpk}.

\begin{remark}
As a note on scope, our work studies IP formulations and path-based 
optimization problems from an algorithmic perspective. We do not handle 
execution details such as workspace discretization, uncertainties, and 
communication issues. These items are eventually to be manged by other 
parts in the system. For example, with proper synchronization and 
feedback based control, solutions generated by the ILP-based MPP 
solver~\cite{YuLav16TRO} can be readily executed on multi-robot hardware
platforms \cite{han2018sear}.
\end{remark}

\vspace*{-2mm}
\section{General Methodology and Path Encoding}\label{sec:method}
\vspace*{-1mm}
Our methodology for path-based optimization problems generally follows a 
straightforward two-step process. In the first step, a basic IP model is constructed 
that ensures only feasible paths are produced. This is the {\em path encoding} step. 
For example, in the case of a multi-robot path planning problem, the IP model must 
produce multiple paths that connect the desired start and goal configurations. 
In the second step, the optimization criteria and additional constraints, e.g., for 
collision avoidance, is enforced. In our experience, the path encoding step has a 
limited variations whereas the second step often has more variations and requires 
some creativity. 

In this section, we provide detailed descriptions of two IP models for encoding 
paths: a {\em base-graph encoding} that works with the original edge set $E$, and 
a {\em time-expanded-graph encoding} that encodes paths on a new graph generated 
by making multiple time-stamped copies of the vertex set $V$. Feasible
assignments to the integer variables in a given model correspond (in a 
one-to-one manner) to all possible paths in the original problem. The 
optimization step will be introduced in the next section. 

\vspace*{-2mm}
\subsection{Base-Graph Encoding}
\vspace*{-1mm}
We define a {\em non-cyclic path} to be a path $P$ that goes through 
any vertex at most once, i.e., $\forall 0 \leq i < j \leq T, p^i \neq p^j$. 
Given $\langle G, x^I, x^G \rangle$, the base graph encoding introduces 
a binary variable $x_{v_i, v_j}$ for each edge $(v_i, v_j) \in E$ to indicate whether $P$ uses 
$(v_i, v_j)$. The following constraints must be satisfied: 
\begin{align}\label{constraint:ig} 
\sum_{v_i \in N(x^I)} x_{x^I, v_i} = \sum_{v_i \in N(x^G)} x_{v_i, x^G} = 1; 
\end{align}
\begin{align}\label{constraint:ig2} 
\sum_{v_i \in N(x^I)} x_{v_i, x^I} = \sum_{v_i \in N(x^G)} x_{x^G, v_i} = 0; 
\end{align}
\begin{align}\label{constraint:in-out}
\sum_{v_j \in N(v_i)} x_{v_i, v_j} = \sum_{v_j \in N(v_i)} x_{v_j, v_i} \leq 1, 
\forall v_i \in V \backslash \{x^I, x^G\}. 
\end{align}
Here, constraint (\ref{constraint:ig}) and (\ref{constraint:ig2}) make $P$ 
starts from $x^I$ and ends at $x^G$. Constraint (\ref{constraint:in-out}) ensures 
that for each vertex, one outgoing edge is used if and only if an 
incoming edge is used. It prevents the path from creating multiple 
branches, and also forces each vertex to appear in $P$ at most once. 

A solution from this formulation could contain {\em subtours}, which are cycles 
formed by edges that are disjoint from $P$ (see Fig.~\ref{fig:subtour}). 
As introduced in \cite{yu2016correlated}, these subtours can be eliminated by creating integer 
variables $3 \leq u_i \leq |V|$ for each $v_i \in V \backslash \{x^I, x^G\}$, and adding 
the following constraint for each pair of vertices $v_i, v_j \in V \backslash \{x^I, x^G\}$:
\begin{align}\label{constraint:subtour}
u_i - u_j + 1 \leq (|V| - 3)(1 - x_{v_i, v_j}). 
\end{align}

\begin{figure}[ht!]
\centering
\vspace{1mm}
\begin{overpic}[scale = 0.75]{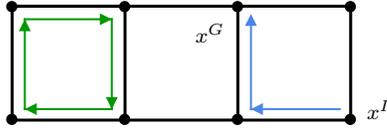}
    \put(103, 2){\footnotesize $x^I$}
    \put(54, 24){\footnotesize $x^G$}
\end{overpic}
\vspace{-2mm}
\caption{Illustration of a subtour in a base-graph encoding formulation. The blue 
lines show a feasible path from $x^I$ to $x^G$. The green cycle is a possible 
undesirable disjoint subtour.}
\label{fig:subtour}
\end{figure}

\begin{proposition}\label{thm:original} 
Given $\langle G, x^I, x^G \rangle$, there exists a bijection between solutions 
to the base-graph encoding IP model and all non-cyclic paths in $G$ from $x^I$ 
to $x^G$.
\end{proposition}
\begin{proof}
($\Rightarrow$) Given a feasible solution of the base-graph encoding IP model, a 
non-cyclic path starts from $x^I$ is constructed by following positive edge 
variables until reaching $x^G$. The path is guaranteed to be feasible since  
{\em (i)} by constraint (\ref{constraint:ig}) and (\ref{constraint:ig2}), the path can 
only start at $x^I$ and end at $x^G$, {\em (ii)} by constraint 
(\ref{constraint:in-out}) and (\ref{constraint:subtour}), the path is composed of 
non-repetitive vertices in $V$ connected by edges in $E$.

($\Leftarrow$) A non-cyclic path $P = (p^0, \dots, p^T)$ can be translated to a feasible 
solution to the IP model by assigning the corresponding edge variables to $1$ and all 
others to $0$. These values satisfies constraint (\ref{constraint:ig}) 
(\ref{constraint:ig2}) (\ref{constraint:in-out}) (\ref{constraint:subtour}) because 
$P$ is a sequence of vertices starts from $x^I$, ends at $x^G$, connected by edges, 
and has no subtours.
\end{proof}

When $x^I = x^G$, we split $x^I$($x^G$) into two vertices $v_{in}, 
v_{out}$ and connect them to all the vertices in $N(x^I)$. A path can then be 
formulated on the graph with vertex set $(V \backslash \{x^I\}) \cup \{v_{in}, v_{out}\}$. 
The resulting path is acyclic in the new graph, but interpreted as a cycled 
path in $G$ and contains $x^I$($x^G$). 

\vspace*{-2mm}
\subsection{Time-Expanded-Graph Encoding}

\begin{figure}[ht!]
\centering
\begin{overpic}[scale = 0.75]{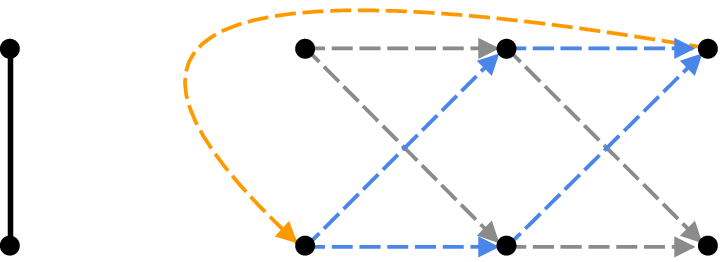}
    \put(5, 0){\footnotesize $x^I$}
    \put(5, 27){\footnotesize $x^G$}
    \put(37, -5){\footnotesize $t = 0$}
    \put(65, -5){\footnotesize $t = 1$}
    \put(92, -5){\footnotesize $t = 2$}
\end{overpic}
\vspace{1mm}
\caption{The construction of a time-expanded graph. The original graph is shown 
on the left, with two vertices connected by an undirected edge. The time-expanded-graph 
with $T = 2$ is shown on the right. Directed edges (blue and grey dashed 
lines) are added to the adjacent vertices in the three copies of $V$. The feedback 
edge is shown as an orange dashed line. The grey edges are removed by reachability tests.}
\label{fig:time-expanded}
\end{figure}

This path encoding method uses a time-expanded-graph instead of the original 
graph $G(V, E)$. Given a fixed time horizon $T \in \mathbb N$, we make $T + 1$ 
copies of $V$, namely $V^0, \dots, V^T$. Neighboring vertices in adjacent time 
steps are then connected by {\em directed edges}: for $1 \leq t \leq T$, the 
edge set between $V^{t - 1}$ and $V^t$ is $\{(v_i^{t - 1}, v_j^t) | \forall 
v_i \in V, v_j \in N(v_i) \cup \{v_i\}\}$. We then add a directed {\em feedback 
edge} connecting $x^{G, T} \in V^T$ to $x^{I, 0} \in V^0$. An example is 
provided in Fig.~\ref{fig:time-expanded}. Similar to base-graph encoding, we 
use a binary variable $x_{v_i, v_j}^t$ to indicate whether $(v_i^{t - 1}, v_j^t)$ 
is used. The variable associated with the feedback edge is $x_{x^G, x^I}^0$. By 
denoting the outgoing (resp. incoming) edges of $v_i$ in this time-expanded 
graph as $N^-(v_i)$ (resp. $N^+(v_i)$), $\forall 1 \leq t \leq T$, constraints 
(\ref{constraint:ig}) (\ref{constraint:ig2}) (\ref{constraint:in-out}) are now 
expressed as:
\begin{align}\label{constraint-feedback} 
\sum_{v_i \in N^+(x^G)} x_{v_i, x^G}^T = x_{x^G, x^I}^0 = 1; 
\end{align}
\begin{align}\label{constraint-flow}
\sum_{v_j \in N^-(v_i)} x_{v_i, v_j}^t = \sum_{v_j \in N^+(v_i)} x_{v_j, v_i}^{t - 1}, 
\ \forall v_i \in V^1 \cup \dots \cup V^T.
\end{align}

The representation of a feasible path in this time-expanded graph is a sequence of 
directed edges that starts from $x^{I, 0}$, travel through exactly one vertex in 
each of $V^0, \dots, V^T$, and finally goes back to $x^{I, 0}$ uses the feedback 
edge. Similarly, a 1-1 solution correspondence exists in this case \cite{YuLav16TRO}. 

\vspace*{-1mm}
\begin{proposition}\label{thm:time-expanded}
There is a bijection function between feasible paths in $G$ with length less than $T + 1$, 
and feasible paths in the $T$-step time-expanded-graph. 
\end{proposition}
\vspace*{-1mm}

In a time-expanded-graph encoding, the length of the path generated is limited 
by the time horizon $T$. The path can contain any vertex for multiple times. 
The number of variables in the time-expanded-graph encoding can be reduced by
performing {\em reachability tests} (see Fig.~\ref{fig:time-expanded}), which 
remove edges not reachable from $x^I$ or $x^G$. Reachability tests do not 
affect completeness and optimality. Scalability can be further improved using 
{\em $k$-way split} \cite{YuLav16TRO}, a divide-and-conquer heuristic that 
splits a problem into smaller sub-problems. The sub-problems require much 
shorter time to solve together than solving a single large problem. 

\vspace*{-2mm}
\subsection{Basic Extensions of the Two Encoding}
\vspace*{-1mm}

Both encoding can handle the case when a robot has multiple initials or goals 
to choose. In base-graph encoding, this is done by adding one virtual vertex, 
and connecting it to all possible initials and goals. In the time-expanded-graph 
encoding formulation, the same objective can be achieved by adding multiple 
feedback edges. 

Both encoding can also be extended to the multi-robot scenario. Which is 
achieved  by simply making one copy of all variables for each robot. The paths 
can then be planned disregarding mutual collisions. As we will show soon, with 
additional constraints, the time-expanded-graph encoding is able to generate 
collision-free paths. 

Here we note that our IP models are solved using Gurobi Solver~\cite{gurobi}. 
All experiments are executed on an Intel\textsuperscript{\textregistered} 
Core\textsuperscript{TM} i7-6900K CPU with 32GB RAM at 2133MHz.

\vspace*{-1mm}
\section{Time-Optimal Multi-Robot Path Planning}\label{sec:applications}
\vspace*{-1mm}
A high-performance time-optimal \mpp IP model was proposed in~\cite{YuLav16TRO} 
using the time-expanded-graph encoding. We review this baseline model and 
introduce two generic new heuristics for trimming the state space. We then 
introduce an updated IP model for MPP with partial solution. 

The basic method \cite{YuLav16TRO} first calculates an underestimated $T$ 
by routing each robot to its goal without considering robot-robot collisions.
This is achieved by running $n$ A* searches and taking the maximum 
path length. Then, an IP model is built and the feasibility is checked: 
for each robot $1 \leq r \leq n$, a set of variables $\{x_{v_i, v_j}^t\}$ 
satisfying constraints (\ref{constraint-feedback}) and (\ref{constraint-flow}) 
is created and renamed as $\{x_{r,v_i, v_j}^t\}$. The method tries to find 
a feasible variable value assignment with the following additional constraints: 
for all $0 \leq t \leq T, v_i \in V^t$,
\begin{align}\label{constraint-mpp-collision-v} 
\sum_{r = 1}^n \sum_{v_j \in N^+(v_i)} x_{r, v_j, v_i}^t \leq 1,
\end{align}
\begin{align}\label{constraint-mpp-collision-e} 
\sum_{r = 1}^n x_{r, v_i, v_j}^t + 
\sum_{r = 1}^n x_{r, v_j, v_i}^t \leq 1, \forall v_j \in N^+(v_i). 
\end{align}
Here, constraint (\ref{constraint-mpp-collision-v}) prevents robots from 
simultaneously occupying the same vertex, while constraint 
(\ref{constraint-mpp-collision-e}) eliminates head-to-head collisions on 
edges in $E$. By Proposition~\ref{thm:time-expanded}, an infeasible model indicates 
that no feasible solution exists in makespan $T$. Then, with $T$ incremented by $1$, 
the model is re-constructed and solved again. Once the model has a feasible 
solution, a solution to \mpp with optimal makespan is then extracted. 

\noindent\textbf{Effective new heuristics}. 
Intuitively, the time to solve an IP is often negatively correlated 
with the number of variables in the model. This suggests that the 
computation time may be reduced if we remove vertices in $V^0, \dots, V^T$ 
which are not likely to be part of the solution path. Recall that when 
calculating the underestimated makespan, for each robot $r$, a shortest path 
$P_r^*$ from $x_r^I$ to $x_i^G$ is obtained. When the graph is not densely 
occupied, we are likely to find a solution by make some minor modifications
to these initial candidate paths. Building on the analysis, we propose two 
new heuristics, based on reachability analysis, to reduce the number 
of variables in the IP model. 

\vspace*{1mm}
\subsubsection{Tubular Neighborhood}
Fixing some parameter $h_t \in \mathbb N$, for robot $r$, the 
time-expanded-graph includes $v_i$ only if $v_i$ is within $h_t$ distance 
from some vertices in $P_r^*$. The reachability region in this case
mimics a {\em tube} around the candidate path. The rationale behind 
the heuristic is that, in general, the actual path a robot takes is not 
likely to significantly deviate from the reference path $P_r^*$.

\vspace*{1mm}
\subsubsection{Reachability Sphere}
Fixing some parameter $h_s \in \mathbb N$, for robot $r$, the 
time-expanded-graph includes $v_i$ only if $v_i$ is within $h_s$ distance 
from the $\lfloor t |P_r^*| / T \rfloor$-th element in $P_r^*$. The basic 
principle behind the reachability sphere heuristic is similar to that 
for the tubular neighborhood heuristic.

The effect of tubular neighborhood and reachability sphere heuristics are 
visualized on the left and right of Fig.~\ref{fig:heuristics}, respectively. 
In this example, we visualize the variables when $t = 0.4T$ and assume for robot $r$, $|P_r^*| = 
0.7 T$. The dashed straight lines indicate $P_r^*$, with their left 
ends $x_r^I$ and right ends $x_r^G$. For clearness, all other vertices are 
not shown. If no heuristic is applied, $V^t$ contains all the vertices in 
the entire canvas. Through reachability tests, a vertex is removed from $V^t$ 
if it is not reachable from $x^I$ in time $t$, or from $x^G$ in time $T - t$. 
Reachability tests remove all the vertices not in the intersection of the 
two arcs centered at $x^I$ and $x^G$, making $V^t$ contain only the vertices 
in the region with red boundaries. The tubular neighborhood heuristic with 
$h_t = 0.12 T$ on the left sub-graph and the reachability sphere heuristic 
with $h_s = 0.18 T$ on the right then removes all the vertices out of the 
shapes with blue borders. The vertices that are copied to $V^t$ are colored 
in orange.

\begin{figure}[ht!]
\centering
\vspace{1mm}
\begin{overpic}[width = 0.8 \linewidth, keepaspectratio]{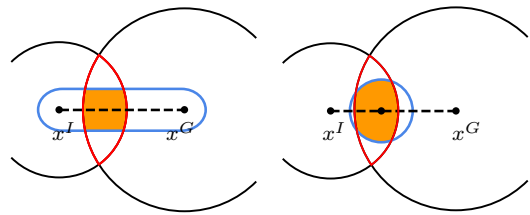}
    \put(8, 14){\footnotesize $x^I$}
    \put(30, 14){\footnotesize $x^G$}
    \put(60, 14){\footnotesize $x^I$}
    \put(85, 14){\footnotesize $x^G$}
\end{overpic}
\vspace{-1mm}
\caption{Illustration of tubular neighborhood (left) and reachability sphere
(right) heuristics.}
\label{fig:heuristics}
\end{figure}

We evaluate how the heuristics influence the performance using randomly 
generated test cases on a $24 \times 18$ grid. As it is shown in 
Fig.~\ref{fig:heuristics-results}, the number of variables is reduced by 
more than $70\%$ when $h_t \leq 2$ or $h_s \leq 2$; the computation time 
is reduced by at least $60\%$ when $n \le 60$. In the best case (the red
curve), the reduction in computation time is over ten fold, which is
very significant. Somewhat to our surprise, for more than $60$ robots, 
reduced variable count in the IP model does not always translate to faster 
solution time. After digging in further, this seems to be related to how 
the Gurobi solver works: we could verify that our heuristics interfere 
with Gurobi's own heuristics when there are too many robots. We report 
that the achieved optimality is not affected by the new heuristics 
for all the test cases. 




\begin{figure}[ht!]
\centering
\includegraphics[keepaspectratio]{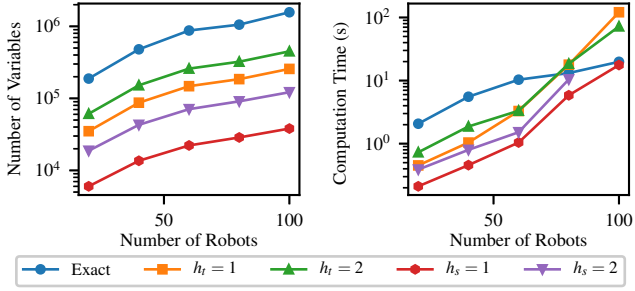}
\vspace{-7mm}
\caption{Evaluation result of the proposed new heuristics for MPP. (left) 
The number of vertices in the IP models versus $n$, (right) computation 
time of the IP method versus $n$, with different heuristic parameters. 
The {\em Exact} entries indicate the original IP method.}
\label{fig:heuristics-results}
\end{figure}


\noindent\textbf{Partial Solutions}.
To accommodate the $k \le n$ constraint in Problem~\ref{prob:mpp}, we 
prioritize using the individual path lengths in determining the time expansion 
parameter $T$. That is, we pick $T$ to the maximum length of the shortest 
$k$ robot paths out of the $n$ paths. For the rest of the $n - k$ robots, we 
allow their goals in the IP model to be in a neighborhood of their respective 
specified goals. This is achieved through the heuristics of adding additional 
feedback edges (see Fig.~\ref{fig:multi-feedback}) for these $(n - k)$ robots. 
Note that now the reachability test should not be applied to the goals of these 
$(n-k)$ robots. 
\begin{figure}[ht!]
\vspace*{-4mm}
\centering
\begin{overpic}[scale = 0.75]{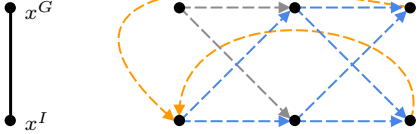}
    \put(5, 0){\footnotesize $x^I$}
    \put(5, 27){\footnotesize $x^G$}
\end{overpic}
\vspace{-1mm}
\caption{Multiple feedback edges for the same example in 
Fig.~\ref{fig:time-expanded} visualized using the same color and dash 
style.}
\label{fig:multi-feedback}
\end{figure}

\begin{figure}[ht!]
\vspace*{-3mm}
\centering
\includegraphics[keepaspectratio, width = 0.6\linewidth]{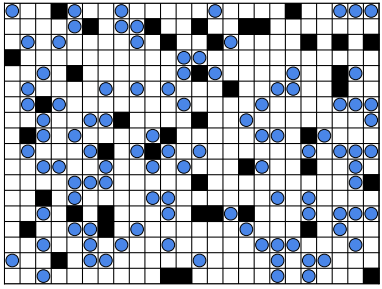}
\vspace{-2mm}
\caption{An example of $24 \times 18$ grid with $10\%$ vertices removed (visualized 
as black cells) and $100$ randomly placed robots (visualized as blue circles).}
\label{fig:2418}
\end{figure}

\begin{figure}[ht!]
\vspace*{-3mm}
\centering
\includegraphics[keepaspectratio, width = 0.8\linewidth]{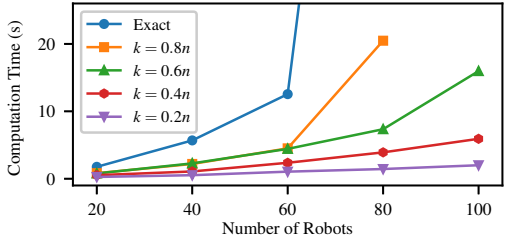}
\vspace{-3mm}
\caption{Computation time of the \mpp partial solver versus $n$, with different $k$ values. 
The {\em Exact} entry indicates the original IP method.}
\label{fig:partial-result}
\end{figure}

In evaluating the general \mpp solver, we use a $24 \times 18$ grid with 
$10\%$ vertices removed to simulate obstacles (see Fig.~\ref{fig:2418} for 
an example). As shown in Fig.~\ref{fig:partial-result}, running time is 
significantly reduced when we only need a partial solution. 

\noindent\textbf{Support for a non-fully labeled problem}. 
We can also adapt the existing IP model to resolve the problem in which a 
group of robots are indistinguishable, i.e., they are assigned a set of 
goals without individual targets. The final configuration of these robots 
in a solution can be an arbitrary permutation of the goals. This is also 
known as the {\em $k$-colored} motion planning problem \cite{SolHal12}. 
To update the model, suppose a group of $m$ robots share $m$ goal vertices, 
we create a single copy of $\{x_{v_i, v_j}^t\}$ for this group, and add 
$m^2$ feedback edges between all the pairs of initials and goals. The 
summation of the variables associated with these edges is set to be $m$.

\vspace*{-1mm}
\section{Multi-Robot Minimum Constraint Removal}\label{sec:mmcr}
\vspace*{-1mm}
As implicitly stated in~\cite{hauser2014minimum}, the state space of \mcr can 
be reduced by building a graph reflecting the coverage of different combinations of obstacles. 
Here, we explicitly construct 
this graph $G_{\mcr}(V_{\mcr}, E_{\mcr})$ such that each element in $V_{\mcr}$ 
is a set of connected vertices that exist in the same set of obstacles. 
Specifically, $V_i \in V_{\mcr}$ satisfies the following constraints: 
{\em (i)} $V_i \subset V$; 
{\em (ii)} $\forall v_j, v_k \in V_i, O \in \mathcal O$, $v_j \in O$ if and only if $v_k \in O$; 
{\em (iii)} $\forall v_j, v_k \in V_i$, there exists a path $P$ from $v_j$ to $v_k$, and 
formed by elements in $V_i$. 
Under this definition, $\bigcup_{V_i \in V_{\mcr}} V_i = V$, and 
$\forall V_i, V_j \in V_{\mcr}$, $V_i \cap V_j = \emptyset$. 
In our implementation, $V_{\mcr}$ is built by {\em iteratively} selecting a 
random $v \in V$ and run {\em breadth first search} from $v$ until all the 
leaf nodes collide with different sets of obstacles from $v$. An undirected 
edge $(V_i, V_j)$ is added to the edge set $E_{\mcr}$ if there exist $v_i 
\in V_i$ and $v_j \in V_j$, such that $(v_i, v_j) \in E$. An example 
$G_{\mcr}$ is shown in Fig.~\ref{fig:mcr-graph}. 

\begin{figure}[ht!]
\centering
\includegraphics[scale = 0.6]{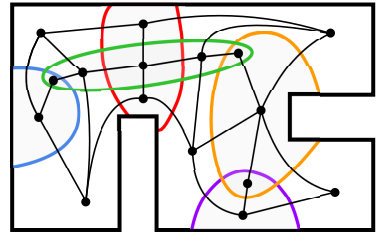}
\vspace*{-2mm}
\caption{$G_{\mcr}$ for the problem instance in Fig~\ref{fig:mmcr-example} 
illustrated using black dots and black lines. Note that each vertex can be considered 
as a high level abstraction of a set of vertices in a more detailed roadmap (i.e. $G$).}
\label{fig:mcr-graph}
\end{figure}

For each robot $1 \leq r \leq n$, we denote $V_r^I \in V_{\mcr}$ as the set contains $x_r^I$ 
and $V_r^G \in V_{\mcr}$ as the set contains $x_r^G$. We assume $V_r^I \neq V_r^G$ since 
otherwise the set of obstacles that need to be removed for robot $r$ is trivially 
$\{O \in \mathcal O | O \cap V_r^I \neq \emptyset\}$. With slight abuse of notation, 
we denote $N(V_i)$ as the neighborhood of $V_i$. Following base-graph encoding, 
for each robot $1 \leq r \leq n$, a set of variables $\{x_{V_i, V_j}\}$ satisfying 
constraints (\ref{constraint:ig}) (\ref{constraint:ig2}) (\ref{constraint:in-out}) are 
created and renamed as $\{x_{r, V_i, V_j}\}$.

In the rest of this section, unless otherwise stated, a {\em path} refers to a 
path in $G_{\mcr}$. We introduce binary variables $\{x_{V_i} | V_i \in 
V_{\mcr}\}$ to indicate if at least one of $n$ paths contains $V_i$, and 
$\{x_{O_i} | O_i \in \mathcal O\}$ to indicate whether $O_i$ must be removed 
to make the paths do not collide with obstacles. The constraints for these 
variables are: 
\begin{align}\label{constraint:mcr-v} 
L \ x_{V_i} \geq \sum_{r = 1}^n \sum_{V_j \in N(V_i)} (x_{r, V_i, V_j} + x_{r, V_j, V_i}); 
\end{align}
\begin{align}\label{constraint:mcr-o} 
L \ x_{O_i} \geq \sum_{V_i \in V_{\mcr}, V_i \subset O_i} x_{V_i}.
\end{align}
Here, $L$ is a large constant. Constraint (\ref{constraint:mcr-v}) 
assigns positive values to $x_{V_i}$ only if $V_i$ is in any path. Constraint 
(\ref{constraint:mcr-o}) ensures that $\forall V_i \in V_{\mcr}$, 
if $x_{V_i} = 1$, then $\{x_{O_j} | O_j \in \mathcal O, V_i \subset O_j\}$ 
are all assigned positive values. The objective for this model is to minimize 
$\sum_{O_i \in \mathcal O} x_{O_i}$, i.e., the number of obstacles to be removed. 
Note that subtour elimination is unnecessary in \mmcr~since neither a subtour 
nor duplicate elements in a path reduce the objective value. Robot-robot 
collision is not considered (but can be if needed). After solving this IP, the 
obstacles to be removed can be extracted from the model, after which an 
actual robot trajectory can be found easily.


We compare the IP approach to two search-based solvers from~\cite{hauser2014minimum}. 
For \mcr, a search node in these solvers is a tuple $\langle v, \mathcal 
O_r\rangle$, where $v$ is the robot's current location, and $\mathcal O_r$ is the 
set of all the obstacles encountered from $v$ tracing back to $x^I$. 
Different pruning methods are then applied to reduce the search space, giving 
rise to exact and greedy solvers. 
Both the exact and greedy solvers can be extended to \mmcr via maintaining 
locations for all robots in the search state. 

The comparison results are given in Fig.~\ref{fig:mcr-result}. The left figure 
is for \mcr evaluated on a $100 \times 100$ grid and the right one for \mmcr 
with $n = 2$ on a $10 \times 10$ grid. All test cases are generated by randomly 
placing arbitrary sized rectangular obstacles in the grids. 

\begin{figure}[ht!]
    \vspace*{-1mm}
    \centering
    \includegraphics[keepaspectratio]{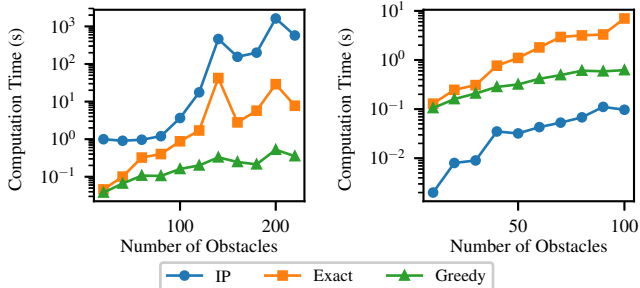}
    \vspace{-7mm}
    \caption{Computation time of different \mcr and \mmcr methods versus 
		the number of obstacles. The left subfigure shows the result for \mcr, 
		while the right     subfigure shows the result for \mmcr. The {\em 
		Exact} and {\em Greedy} entries indicate the exact and greedy search-based 
		solvers, respectively.}
    \label{fig:mcr-result}
\end{figure}

We observe that the exact IP method runs slower than the exact discrete 
search-based method for \mcr but significantly outperforms both 
search-based methods on \mmcr. We believe the reason is that with a 
single robot in a grid-based environment, the problem is not ``complex'' 
enough for IP-based method to utilize its structural advantages. Further 
evaluation shows that our IP method can solve a problem on a $50 \times 
50$ grid with $100$ robots and $100$ obstacles in around $5$ seconds, 
while the search-based solvers would fail to complete in such a case. 

As may be observed in Fig.~\ref{fig:mcr-result}, left, the computation time 
does not grow monotonically as the number of obstacles increases, even though 
the data is already averaged over $20$ instances. This is due to the 
NP-hardness of \mcr and there could be some random instances that are 
particularly hard to solve. However, note that the relative performance 
difference between methods remains consistent.


\vspace*{-4mm}
\section{Reward Collection Problems}\label{sec:rcp}
In this section, we examine some complex reward collection problems (RCPs) 
to further demonstrate the flexibility of our IP methodology. The initial work 
on \qcop \cite{yu2016correlated} and \otp \cite{yu2014optimal} used base-graph 
encoding. We show here these problems can also be solved using 
time-expanded-graph encoding to gain better computational performance. 

In \qcop, the initial and goal vertices can be freely chosen from fixed sets  
$X^I \subset V, X^G \subset V$. To handle this, we create a virtual vertex $u$ 
and add directed edges from $u$ to all the candidacy initial vertices in $V^0$, 
and from all the possible goal vertices in $V^0, \dots, V^T$ to $u$. 
Specifically, these directed edges are added to the time-expanded-graph: 
$\{(u, v^0) | v \in X^I\} \cup \{(v^t, u) | v \in X^G, 0 \leq t \leq T\}$. 
For a fixed time horizon $T$, constraint (\ref{constraint-feedback}) is now 
expressed as 
\begin{align}\label{constraint-feedback-qcop} 
\sum_{v_i \in N^-(u)} x_{u, v^i}^0 = \sum_{t = 0}^T \sum_{v_i \in N^+(u)} x_{v_i, u}^t = 1. 
\end{align}
Using a binary variable $x_{v_i}$ to indicate whether $v_i \in V$ is in the 
path, the constraint to assign a correct value to $x_{v_i}$ is
\begin{align}\label{constraint-qcop-v} 
x_{v_i} \leq \sum_{t = 0}^T \sum_{v_j \in N^+(v_i^t)} x_{v_j, v_i}^t. 
\end{align}
With binary variables indicating whether vertices and edges are used, the 
objective value (\ref{qcop-r}) (to be maximized) is directly encoded into 
the IP model. The time consumption requirement (\ref{qcop-t}) is added as 
a constraint.

The time-expanded model for \otp is similar to \qcop with one extra set of 
constraints $x_{u, v^i}^0 = \sum_{t = 0}^T x_{v_i, u}^t, \forall v_i \in X^I$
to ensure $x^I = x^G$, and an additional set of non-integer non-negative 
variables $t_i$ to indicate how long the robot stays at $v_i$. The constraint 
$t_i \leq L \ x_{v_i}$ ensures that $t_i$ is a positive value only if $v_i$ 
is in the path where $L$ is a large constant. 

We use variable size grid settings similar to those from 
\cite{yu2016correlated,yu2014optimal} for evaluation. We let $X^I$ contain $2$ 
randomly selected vertices. In \qcop, $x^G = V$; a random reward weight $r$ is 
assigned to each vertex. In \otp, we assume $R_i(t_i)$ are linear functions with 
random positive coefficients. We observe from the result 
(Fig.~\ref{fig:rcp-result}) that the time-expanded-graph encoding is always 
competitive and performs significantly better as the size of the problem gets 
larger. 

\begin{figure}[ht!]
\centering
\includegraphics[keepaspectratio]{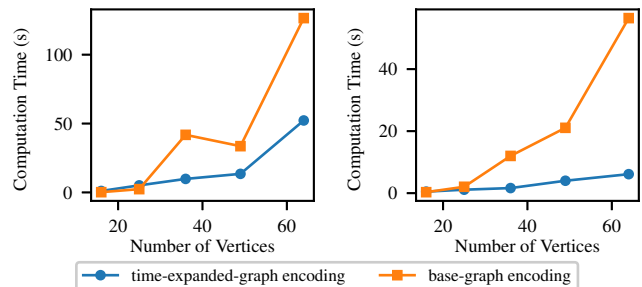}
\vspace{-7mm}
\caption{Computation time of the two path encoding methods versus the number of vertices in $G$: 
(left) \qcop result, the fluctuation for base-graph encoding is due to its sensitivity to grid aspect ratio; 
(right) \otp result.
}
\label{fig:rcp-result}
\end{figure}

\section{Discussion and Conclusion}\label{sec:conclusion}
In this work, building on previous efforts, we propose a two-step 
integer programming (IP) methodology for solving path-based optimization 
problems. The approach is applicable to a variety of computationally 
hard problems in robotics involving filtering through a huge set of 
candidate paths. Although simple to use, harnessing the power of heavily 
optimized solvers, the IP method comes with performance that is often 
competitive or even beats conventional methods. We point out two major 
strengths that come with the IP solution method: {\em (i)} due to its 
ease of use, the time that is required for developing a solution can be 
greatly reduced, and {\em (ii)} the method can provide reference optimal 
solutions to help speed up the design of conventional algorithms. With 
the study, which provides principles and many best practices for IP 
model construction for path-based optimization, we hope to promote the 
adoption of the method as a first choice when practitioners in robotics 
attack a new problem.

{\small
\bibliographystyle{IEEEtran}
\bibliography{all}
}

\newpage
\section*{Appendix: \\ Detailed IP Model Construction Practice}\label{section:appendix}

In this section, we provide implementation details of our IP formulations 
using some simple problem examples. To better show the IP structures, we 
\begin{itemize}
    \item do not fully apply the reachability test, and 
    \item do not remove redundant variables and constraints.
\end{itemize}

\subsection{\mpp with Partial Solution}

\begin{figure}[htp]
    \centering
    \vspace*{3mm}
    \begin{overpic}[scale = 0.75]{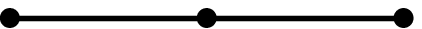}
        \put(5,  -8){\footnotesize $v_0$}
        \put(52, -8){\footnotesize $v_1$}
        \put(99,-8){\footnotesize $v_2$}
    \end{overpic}
    \vspace{2mm}
    \caption{A simple MPP instance with three verices $\{v_0, v_1, v_2\}$.}
    \label{fig:app-mpp}
\end{figure}

Consider $G$ as shown in Fig~\ref{fig:app-mpp}, with problem parameters 
$n = 2, k = 1, x^I_1 = v_0, x^G_1 = v_1, x^I_2 = v_1, x^G_2 = v_2$. 
Using a time-expanded encoding method, the binary variables in our IP formulation for $T = 1$ are: 
\[x^0_{1, v_0, v_0}, x^0_{1, v_1, v_0}, x^0_{2, v_0, v_1}, x^0_{2, v_1, v_1}, x^0_{2, v_2, v_1},\]
\[x^1_{1, v_0, v_0}, x^1_{1, v_0, v_1}, x^1_{2, v_1, v_0}, x^1_{2, v_1, v_1}, x^1_{2, v_1, v_2},\]
where the first $5$ variables denote the usage of feedback edges, 
and the last $5$ variables denote the usage of time-expanded edges generated from $G$. 

The objective function is:
\[\text{maximize} \quad x^0_{1, v_1, v_0} + x^0_{2, v_1, v_0}.\]

The constraints in the IP formulation are:

\noindent By constraint (\ref{constraint-feedback}),
\[x^0_{1, v_1, v_0} = x^1_{1, v_0, v_1}  = 1, x^0_{2, v_2, v_1} = x^1_{2, v_1, v_2} = 1.\]
By constraint (\ref{constraint-flow}),
\[x^1_{1, v_0, v_0} + x^1_{1, v_0, v_1} = x^0_{1, v_0, v_0} + x^0_{1, v_1, v_0},\] 
\[x^1_{2, v_1, v_2} + x^1_{2, v_1, v_1} + x^1_{2, v_1, v_0} = x^0_{2, v_2, v_1} + x^0_{2, v_1, v_1} + x^0_{2, v_0, v_1}.\] 
By constraint (\ref{constraint-mpp-collision-v}),
\[x^0_{1, v_0, v_0} + x^0_{1, v_1, v_0} \leq 1, \]
\[x^0_{2, v_0, v_1} + x^0_{2, v_1, v_1} + x^0_{2, v_2, v_1} \leq 1, \]
\[x^1_{1, v_0, v_0} + x^1_{2, v_1, v_0} \leq 1, \]
\[x^1_{1, v_0, v_1} + x^1_{2, v_1, v_1} \leq 1, \]
\[x^1_{2, v_1, v_2} \leq 1. \]
By constraint (\ref{constraint-mpp-collision-e}):
\[x^0_{1, v_0, v_0} \leq 1, \]
\[x^0_{2, v_0, v_1} + x^0_{1, v_1, v_0} \leq 1, \]
\[x^0_{1, v_1, v_1} \leq 1, \]
\[x^0_{2, v_2, v_1} \leq 1, \]
\[x^1_{1, v_0, v_0} \leq 1, \]
\[x^1_{1, v_0, v_1} + x^1_{2, v_1, v_0} \leq 1, \]
\[x^1_{2, v_1, v_1} \leq 1. \]
\[x^1_{2, v_1, v_2} \leq 1. \]
Note that for a partial solution, the objective function does not need to 
be optimized; we can terminate the IP solution process when
\[x^0_{1, v_1, v_0} + x^0_{2, v_1, v_0} \geq k = 1. \]

\subsection{\mpp Heuristics}
Continuing with the previous example, when we use tubular neighborhood with 
$h_t = 0$, variables $x^0_{2, v_0, v_1}$ and $x^1_{2, v_1, v_0}$ will both 
be set to $0$ since the shortest path for robot $2$ does not include $v_0$. 
When rechability spheres with $h_s = 0$ is used, $4$ more variables 
$x^0_{1, v_0, v_0}, x^1_{1, v_0, v_0}, x^0_{2, v_1, v_1}, x^1_{2, v_1, v_1}$ 
will be set to $0$. For this specific problem instance, these modifications do not 
affect the solution feasibility and optimality. 

Note that although a good reachability test will remove all the $6$ variables mentioned 
above, on a larger problem instance the variables affected by each heuristic are 
generally not subsets of each other.

\subsection{\mmcr}

\begin{figure}[htp]
    \centering
    \vspace*{1mm}
    \begin{overpic}[scale = 0.75]{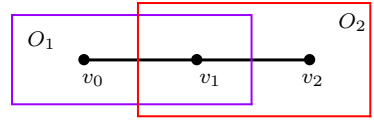}
        \put(20, 10){\footnotesize $v_0$}
        \put(52, 10){\footnotesize $v_1$}
        \put(80, 10){\footnotesize $v_2$}
        \put(5, 20){\footnotesize $O_1$}
        \put(90, 25){\footnotesize $O_2$}
    \end{overpic}
    \vspace{2mm}
    \caption{A simple (M)MCR instance with three verices $\{v_0, v_1, v_2\}$ 
    and two obstacles $O_1 = \{v_0, v_1\}$, $O_2 = \{v_1, v_2\}$.}
    \label{fig:app-mcr}
\end{figure}

As visualized in Fig~\ref{fig:app-mpp}, we add two obstacles into the previous \mpp 
example to make it a \mmcr example. With the same problem parameters 
$n = 2, x^I_1 = v_0, x^G_1 = v_1, x^I_2 = v_1, x^G_2 = v_2$. 
Using the base graph formulation, the binary variables are:
\[x_{1, v_0, v_1}, x_{1, v_1, v_0}, x_{1, v_1, v_2}, x_{1, v_2, v_1}, \]
\[x_{2, v_0, v_1}, x_{2, v_1, v_0}, x_{2, v_1, v_2}, x_{2, v_2, v_1}, \]
\[x_{v_0}, x_{v_1}, x_{v_2}, x_{O_1}, x_{O_2}.\]
The objective function is to minimize the number of obstacles that 
intersect with paths:
\[\text{minimize} \quad x_{O_1} + x_{O_2}.\]
By constraint (\ref{constraint:ig}), 
\[x_{1, v_0, v_1} = x_{1, v_0, v_1} + x_{1, v_2, v_1} = 1, \]
\[x_{2, v_1, v_0} + x_{2, v_1, v_2} = x_{2, v_1, v_2} = 1. \]
By constraint (\ref{constraint:ig2}), 
\[x_{1, v_1, v_0} = x_{1, v_1, v_0} + x_{1, v_1, v_2} = 0, \]
\[x_{2, v_0, v_1} + x_{2, v_2, v_1} = x_{2, v_2, v_1} = 0. \]
By constraint (\ref{constraint:in-out}), 
\[x_{1, v_2, v_1} = x_{1, v_1, v_2} \leq 1, \]
\[x_{2, v_0, v_1} = x_{2, v_1, v_0} \leq 1. \]

Note that constraint (\ref{constraint:subtour}) is not necessary in \mmcr since 
additional subtours cannot decrease the value of the objective function. 

\noindent By constraint (\ref{constraint:mcr-v}), 
\[L x_{v_0} \geq x_{1, v_0, v_1} + x_{1, v_1, v_0} + x_{2, v_0, v_1} + x_{2, v_1, v_0}, \]
\begin{align*}
    L x_{v_1} \geq \, & x_{1, v_1, v_0} + x_{1, v_0, v_1} + x_{1, v_1, v_2} + x_{1, v_2, v_1} + \\
    & x_{2, v_1, v_0} + x_{2, v_0, v_1} + x_{2, v_1, v_2} + x_{2, v_2, v_1}, 
\end{align*}
\[L x_{v_2} \geq x_{1, v_2, v_1} + x_{1, v_1, v_2} + x_{2, v_2, v_1} + x_{2, v_1, v_2}. \]
By constraint (\ref{constraint:mcr-o}), 
\[L x_{O_1} \geq x_{v_0} + x_{v_1}, \]
\[L x_{O_2} \geq x_{v_1} + x_{v_2}. \]

\subsection{\rcp}

For this section, we only demonstrate an implementation example for \otp since the formulation can be 
roughly considered as \qcop with extra variables and constraints. 
We continue to use the graph in Fig.~\ref{fig:app-mpp}. 
Consider the three vertices are associated with three linear reward functions 
$R_0, R_1, R_2$; the robot starts from vertex $v_0$; the cost of edge $(v_0, v_1)$ 
and $(v_1, v_2)$ are $c_{01}$ and $c_{12}$; the maximum cost is $c^*$; 
$X^I = X^G = \{v_0, v_1\}$. 
Using a time-expanded encoding method and with the additional virtual vertex $u$, 
the binary variables in our IP formulation for $T = 2$ are: 
\[x^0_{u, v_0}, x^0_{u, v_1}, x^2_{v_0, u}, x^2_{v_1, u}, \]
\[x^1_{v_0, v_0}, x^1_{v_0, v_1}, x^1_{v_1, v_0}, x^1_{v_1, v_1}, x^1_{v_1, v_2}, 
x_{v_0}, x_{v_1}, x_{v_2}. \]
The three additional continuous variables are:
\[t_0 \geq 0, t_1 \geq 0, t_2 \geq 0.\]
The objective function is:
\[\text{maximize} \quad R_0(t_0) + R_1(t_1) + R_2(t_2). \]
By constraint (\ref{otp-t}), 
\[c_{01} * (x^1_{v_0, v_1} + x^1_{v_1, v_0}) + c_{12} * x^1_{v_1, v_2} + t_0 + t_1 + t_2 \leq c^*. \]
By constraint (\ref{constraint-feedback-qcop}) (instead of (\ref{constraint-feedback})), 
\[x^0_{u, v_0} + x^0_{u, v_1} = x^2_{v_0, u} + x^2_{v_1, u} = 1, \]
By constraint (\ref{constraint-flow}), 
\[x^1_{v_0, v_0} + x^1_{v_0, v_1} = x^0_{u, v_0}, \]
\[x^1_{v_1, v_0} + x^1_{v_1, v_1} + x^1_{v_1, v_2} = x^0_{u, v_1}, \]
\[x^2_{v_0, u} = x^1_{v_0, v_0} + x^1_{v_1, v_0}, \]
\[x^2_{v_1, u} = x^1_{v_0, v_1} + x^1_{v_1, v_1}. \]
By constraint (\ref{constraint-qcop-v}), 
\[x_{v_0} \leq x^0_{u, v_0} + x^1_{v_0, v_0} + x^1_{v_1, v_0}, \]
\[x_{v_1} \leq x^0_{u, v_1} + x^1_{v_0, v_1} + x^1_{v_1, v_1}, \]
\[x_{v_2} \leq x^1_{v_1, v_2}.\]
Additional constraints for \otp are:
\[x^0_{u, v_0} = x^2_{v_0, u}, x^0_{u, v_1} = x^2_{v_1, u}\]
\[t_0 \leq L x_{v_0}, t_1 \leq L x_{v_1}, t_2 \leq L x_{v_2}. \]

\end{document}